\documentclass[sigconf]{acmart}

\AtBeginDocument{%
  \providecommand\BibTeX{{%
    \normalfont B\kern-0.5em{\scshape i\kern-0.25em b}\kern-0.8em\TeX}}}

\copyrightyear{2022}
\acmYear{2022}
\setcopyright{rightsretained}
\acmConference[WSDM '22]{Proceedings of the Fifteenth ACM
International Conference on Web Search and Data Mining}{February
21--25, 2022}{Tempe, AZ, USA}
\acmBooktitle{Proceedings of the Fifteenth ACM International Conference
on Web Search and Data Mining (WSDM '22), February 21--25, 2022,
Tempe, AZ, USA}
\acmDOI{10.1145/3488560.3498439}
\acmISBN{978-1-4503-9132-0/22/02}

\usepackage{titlesec}
\usepackage{caption}
\usepackage{subcaption}
\usepackage{booktabs}
\usepackage{multirow}
\usepackage{multicol}
\usepackage{xspace}
\usepackage{dsfont}
\usepackage{bbm}
\usepackage{bm}
\usepackage{pifont}
\usepackage{amsmath}
\usepackage{amsthm}
\usepackage{mathtools}
\usepackage{stackengine,scalerel}
\usepackage{enumitem}
\usepackage{wrapfig}
\usepackage{caption}
\usepackage{cancel}
\setlist{leftmargin=5mm}

\usepackage{hyperref}
\usepackage[capitalise, noabbrev]{cleveref}

\titleformat{\paragraph}[runin]{\bfseries\itshape}{\theparagraph}{1em}{}
\titlespacing*{\paragraph}{0pt}{3.25ex plus 1ex minus .2ex}{1em}

\makeatletter \renewcommand\paragraph{\@startsection{paragraph}{4}{\z@} {2mm \@plus1ex \@minus.2ex} {-0.7em} {\normalfont\normalsize\itshape}} \makeatother

\newcommand{\listheader}[1]{\item \emph{#1} }

\newcommand\overstar[1]{\ThisStyle{\ensurestackMath{%
  \setbox0=\hbox{$\SavedStyle#1$}%
  \stackengine{0pt}{\copy0}{\kern.2\ht0\smash{\SavedStyle*}}{O}{c}{F}{T}{S}}}}

\newtheorem{theorem}{Theorem}[section]

\newtheorem{proposition}[theorem]{Proposition}

\newcommand{\ie}[0]{\emph{i.e.}\xspace}
\newcommand{\etc}[0]{\emph{etc.}\xspace}
\newcommand{\eg}[0]{\emph{e.g.}\xspace}
\newcommand{\vs}[0]{\emph{vs.}\xspace}
\newcommand{\wrt}[0]{\emph{w.r.t.}\xspace}

\newcommand{\dataset}[0]{$\mathcal{D}$\xspace}

\newcommand{\argmin}[1]{\underset{#1}{\operatorname{arg}\,\operatorname{min}}\;}

\newcommand{\oracle}{\textsc{Data-genie}\xspace}
\newcommand{\sampler}{\textsc{SVP-CF}\xspace}
\newcommand{\samplerprop}{\textsc{SVP-CF-Prop}\xspace}

\newcommand{\EE}{\operatornamewithlimits{\mathbb{E}}} 

\allowdisplaybreaks 

\begin{document}
\fancyhead{}

\title{On Sampling Collaborative Filtering Datasets}

\author{Noveen Sachdeva}
\email{nosachde@ucsd.edu}
\affiliation{%
  \institution{University of California, San Diego}
  \city{La Jolla, CA}
  \country{USA}
  \postcode{92122}
}

\author{Carole-Jean Wu}
\email{carolejeanwu@fb.com}
\affiliation{%
  \institution{Facebook AI Research}
  \city{Cambridge, MA}
  \country{USA}
}

\author{Julian McAuley}
\email{jmcauley@ucsd.edu}
\affiliation{%
  \institution{University of California, San Diego}
  \city{La Jolla, CA}
  \country{USA}
  \postcode{92122}
}

\renewcommand{\shortauthors}{Sachdeva, Wu, and McAuley}

\begin{abstract}
  We study the practical consequences of dataset sampling strategies on the ranking performance of recommendation algorithms. Recommender systems are generally trained and evaluated on \emph{samples} of larger datasets. Samples are often taken in a na\"ive or ad-hoc fashion: \eg by sampling a dataset randomly or by selecting users or items with many interactions. As we demonstrate, commonly-used data sampling schemes can have significant consequences on algorithm performance. Following this observation, this paper makes three main contributions: (1) \emph{characterizing} the effect of sampling on algorithm performance, in terms of algorithm and dataset characteristics (\eg sparsity characteristics, sequential dynamics, \etc); (2) designing \sampler, which is a data-specific sampling strategy, that aims to preserve the relative performance of models after sampling, and is especially suited to long-tailed interaction data; and (3) developing an \emph{oracle}, \oracle, which can suggest the sampling scheme that is most likely to preserve model performance for a given dataset. The main benefit of \oracle is that it will allow recommender system practitioners to quickly prototype and compare various approaches, while remaining confident that algorithm performance will be preserved, once the algorithm is retrained and deployed on the complete data. Detailed experiments show that using \oracle, we can discard upto $5\times$ more data than any sampling strategy with the same level of performance.
\end{abstract}




\maketitle

\section{Introduction} 
Representative \emph{sampling} of collaborative filtering (CF) data is a crucial problem from numerous stand-points and can be performed at various levels: (1) mining hard-negatives while training complex algorithms over massive datasets \cite{eclare, sampling_cf_nn}; (2) down-sampling the item-space to estimate expensive ranking metrics \cite{sampled_metrics, castells_sampling}; and (3) 
reasons like easy-sharing, fast-experimentation, mitigating the significant environmental footprint of training resource-hungry machine learning models \cite{google_emissions, wu2021sustainable, facebook_emissions, green_ai}. In this paper, we are interested in finding a sub-sample of a dataset which has minimal effects on model utility evaluation \ie an algorithm performing well on the sub-sample should also perform well on the original dataset.

Preserving \emph{exactly} the same levels of performance on sub-sampled data over metrics, such as MSE and AUC, is a challenging problem. A simpler yet useful problem is accurately preserving the \emph{ranking} or relative performance of different algorithms on sub-sampled data. For example, a sampling scheme that has low bias but high variance in preserving metric performance values has less utility than a different sampling scheme with high amounts of bias but low variance, since the overall algorithm ranking is still preserved. 

Performing 
ad-hoc sampling such as randomly removing interactions, or making dense subsets by removing users \emph{or} items with few interactions \cite{sigir20} can have adverse downstream repercussions. For example, sampling only the head-portion of a dataset---from a fairness and inclusion perspective---is inherently biased against minority-groups and benchmarking algorithms on this biased data is 
likely to propagate the 
original sampling bias.
Alternatively,
simply from 
a model performance view-point,
accurately retaining the relative performance of different recommendation algorithms on much smaller sub-samples is a challenging research problem in itself.

Two prominent directions toward better and more representative sampling of CF data are: (1) designing principled sampling strategies, especially for user-item interaction data; and (2) analyzing the performance of different sampling strategies, in order to better grasp which sampling scheme performs ``better'' for which types of data. \emph{In this paper,} we explore both directions through the lens of expediting the recommendation algorithm development cycle by:
\begin{itemize}
    \item Characterizing the efficacy of \emph{sixteen} different sampling strategies in accurately benchmarking various kinds of recommendation algorithms on smaller sub-samples.

    \item Proposing a sampling strategy, \sampler, which dynamically samples the ``toughest'' portion of a CF dataset. \sampler is tailor-designed to handle the inherent data heterogeneity and missing-not-at-random properties in user-item interaction data.
    
    \item Developing \oracle, which analyzes the \emph{performance} of different sampling strategies. Given a dataset sub-sample, \oracle can directly estimate the likelihood of model performance being preserved on that sample. 
\end{itemize}
%
Ultimately, our experiments reveal that \sampler outperforms all other sampling strategies and can accurately benchmark recommendation algorithms with roughly $50-60\%$ of the original dataset size. Furthermore, by employing \oracle to dynamically select the best sampling scheme for a dataset, we are able to preserve model performance with only $10\%$ of the initial data, leading to a net $5.8\times$ training time speedup.

\section{Related Work} \paragraph{Sampling CF data.} Sampling in CF-data has been a popular choice for three major scenarios. Most prominently, sampling is used for mining hard-negatives while training recommendation algorithms. Some popular approaches include random sampling; using the graph-structure 
\cite{pinsage, eclare}; and ad-hoc techniques like similarity search \cite{slice}, stratified sampling \cite{sampling_cf_nn}, \etc 
Sampling is also generally employed for evaluating recommendation algorithms by estimating expensive to compute metrics like Recall, nDCG, \etc \cite{sampled_metrics, castells_sampling}. Finally, sampling is also used to create smaller sub-samples of 
a big dataset
for reasons like fast experimentation, benchmarking different algorithms, privacy concerns, \etc However, the consequences of different samplers on any of these downstream applications is under-studied, and is the main research interest of this paper. 

\paragraph{Coreset selection.} Closest to our work, a coreset is loosely defined as a subset of data-points that maintains a similar ``quality'' as the full dataset for subsequent model training. Submodular approaches try to optimize a function $f : \mathbf{V} \mapsto \mathcal{R}_+$ which measures the utility of a subset $\mathbf{V} \subseteq \mathbf{X}$, and use it as a proxy to select the best coreset \cite{coreset_1}. More recent works treat coreset selection as a bi-level optimization problem \cite{coreset_bilevel, coreset_bilevel_2} and directly optimize for the best coreset for a given downstream task. Selection-via-proxy \cite{svp} is another technique which employs a base-model as a proxy to tag the importance of each data-point. Note, however, that 
most existing
coreset selection approaches were designed primarily for classification data, whereas adapting them for CF-data is non-trivial because of: (1) the inherent data heterogeneity; the (2) wide range of recommendation metrics; and (3) the prevalent missing-data characteristics.

\paragraph{Evaluating sample quality.} The quality of a dataset sample, if estimated correctly is of high interest for various applications. Short of being able to evaluate the ``true'' utility of a sample, one generally resorts to either retaining task-dependent characteristics \cite{evaluate_sample_quality_1} \emph{or} employing universal, handcrafted features like a social network's hop distribution, eigenvalue distribution, \etc \cite{large_graphs} as meaningful proxies. Note that evaluating the sample quality with a limited set of handcrafted features might introduce bias in the sampled data, depending on the number and quality of such features.

\section{Sampling Collaborative Filtering Datasets} Given our motivation of quickly benchmarking recommendation algorithms, we now aim to \emph{characterize} the performance of various commonly-used sampling strategies. We loosely define the performance of a sampling scheme as 
its
ability in effectively retaining the performance-ranking of different recommendation algorithms on the full \vs sub-sampled data. In this section, we start by discussing the different recommendation feedback scenarios we consider, along with a representative sample of popular recommendation algorithms that we aim to efficiently benchmark. We then examine popular data sampling strategies, followed by proposing a novel, proxy-based sampling strategy (\sampler) that is especially suited for sampling representative subsets from long-tail CF data.

\subsection{Problem Settings \& Methods Compared} \label{feedback_types} \label{algorithms}
To give a representative sample of typical recommendation scenarios, we consider three different user feedback settings. In \emph{explicit feedback}, each user $u$ gives a numerical rating $r^u_i$ to each interacted item $i$; the model must predict these ratings for novel (test) user-item interactions. Models from this class are evaluated in terms of the Mean Squared Error (MSE) of the predicted ratings. Another scenario we consider is \emph{implicit feedback}, where the interactions for each user are only available for positive items (\eg clicks or purchases), whilst all non-interacted items are considered as negatives. We employ the AUC, Recall@$100$, and nDCG@$10$ metrics to evaluate model performance for implicit feedback algorithms. Finally, we also consider \emph{sequential feedback}, where each user $u$ interacts with an ordered sequence of items $\mathcal{S}^u = (\mathcal{S}^u_1, \mathcal{S}^u_2, \ldots, \mathcal{S}^u_{|\mathcal{S}^u|})$ such that $\mathcal{S}^u_i \in \mathcal{I}$ for all $i \in \{1, \ldots, |\mathcal{S}^u|\}$. Given $\mathcal{S} = \{ \mathcal{S}^u ~|~ \forall u \in \mathcal{U} \}$, the goal is to identify the \emph{next-item} for each sequence $\mathcal{S}^u$ that each user $u$ is most likely to interact with. We use the same metrics as in implicit feedback settings. Note that following recent warnings against sampled metrics for evaluating recommendation algorithms \cite{sampled_metrics, castells_sampling}, we compute both Recall and nDCG by ranking \emph{all} items in the dataset. Further specifics about the datasets used, pre-processing, train/test splits, \etc are discussed in-depth in \cref{main_exp}. 

Given the diversity of the scenarios discussed above, there are numerous relevant recommendation algorithms. We use the following seven recommendation algorithms, intended to represent the state-of-the-art and standard baselines:
\begin{itemize}
    \listheader{PopRec:} A na\"ive baseline that simply ranks items according to overall train-set popularity. Note that this method is unaffected by the user for which items are being recommended, and has the \emph{same global ranking} of all items.
    
    \listheader{Bias-only:} Another simple baseline that assumes no interactions between users and items. Formally, it learns: (1) a global bias $\alpha$; (2) scalar biases $\beta_u$ for each user $u \in \mathcal{U}$; and (3) scalar biases $\beta_i$ for each item $i \in \mathcal{I}$. Ultimately, the rating/relevance for user $u$ and item $i$ is modeled as $\hat{r}^u_i = \alpha + \beta_u + \beta_i$.

    \listheader{Matrix Factorization (MF) \cite{mf}:} Represents both users and items in a common, low-dimensional latent-space by factorizing the user-item interaction matrix. Formally, the rating/relevance for user $u$ and item $i$ is modeled as $\hat{r}^u_i = \alpha + \beta_u + \beta_i + \gamma_u \cdot \gamma_i$ where $\gamma_u, \gamma_i \in \mathbb{R}^d$ are learned latent representations. 
    
    \listheader{Neural Matrix Factorization (NeuMF) \cite{neural_mf}:} Leverages the representation power of deep neural-networks to capture non-linear correlations between user and item embeddings. Formally, the rating/relevance for user $u$ and item $i$ is modeled as $\hat{r}^u_i = \alpha + \beta_u + \beta_i + f(\gamma_u ~||~ \gamma_i ~||~ \gamma_u \cdot \gamma_i)$ where $\gamma_u, \gamma_i \in \mathbb{R}^d$, `||' represents the concatenation operation, and $f : \mathbb{R}^{3d} \mapsto \mathbb{R}$ represents an arbitrarily complex neural network. 
    
    \listheader{Variational Auto-Encoders for Collaborative Filtering (MVAE) \cite{mvae}:} Builds upon the Variational Auto-Encoder (VAE) \cite{vae} framework to learn a low-dimensional representation of a user's consumption history. More specifically, MVAE encodes each user's bag-of-words consumption history using a VAE and further decodes the latent representation to obtain the completed user preference over all items.
    
    \listheader{Sequential Variational Auto-Encoders for Collaborative Filtering (SVAE) \cite{svae}:} A sequential algorithm that combines the temporal modeling capabilities of a GRU \cite{gru} along with the representation power of VAEs. Unlike MVAE, SVAE uses a GRU to encode the user's consumption sequence followed by a multinomial VAE at each time-step to model the likelihood of the next item. 
    
    \listheader{Self-attentive Sequential Recommendation (SASRec) \cite{sasrec}:} Another sequential algorithm that relies on the sequence modeling capabilities of self-attentive neural networks \cite{self_attention} to predict the occurance of the 
    next item
    in a user's consumption sequence. To be precise, given a user $u$ and 
    their
    time-ordered consumption history  $\mathcal{S}^u = (\mathcal{S}^u_1, \mathcal{S}^u_2, \ldots, \mathcal{S}^u_{|\mathcal{S}^u|})$, SASRec first applies self-attention on $\mathcal{S}^u$ followed by a series of non-linear feed-forward layers to finally obtain the next item likelihood.
\end{itemize}
We also list 
models and metrics for each of the three different CF-scenarios in \cref{model_scenario_table}. Since bias-only, MF, and NeuMF can be trained for all three CF-scenarios, we optimize them using the regularized least-squares regression loss for explicit feedback, and the pairwise-ranking (BPR \cite{bpr}) loss for implicit/sequential feedback. Note however that the aforementioned algorithms are only intended to be a representative sample of a 
wide pool
of recommendation algorithms, and in our pursuit to benchmark recommender systems faster, we are primarily concerned with the \emph{ranking} of different algorithms on the full dataset \vs a smaller sub-sample.

\begin{table*}[!ht]
    \begin{small} 
    \begin{center}
                
                
        \begin{tabular}{c | c c c c c c c | c c c c}
            \toprule
            \multirow{3}{*}{CF-scenario} & \multicolumn{7}{c|}{\emph{Algorithm}} & \multicolumn{4}{c}{\emph{Metric}} \\
            & \multicolumn{7}{c|}{} & \multicolumn{4}{c}{} \\
            & Bias-only & MF & NeuMF & PopRec & MVAE & SVAE & SASRec & MSE & AUC & Recall@k & nDCG@k \\ \midrule
            Explicit & Yes & Yes & Yes & $\times$ & $\times$ & $\times$ & $\times$ & Yes & $\times$ & $\times$ & $\times$ \\[0.6mm]
            Implicit & Yes & Yes & Yes & Yes & Yes & $\times$ & $\times$ & $\times$ & Yes & Yes & Yes \\[0.6mm]
            Sequential & Yes & Yes & Yes & Yes & Yes & Yes & Yes & $\times$ & Yes & Yes & Yes \\[0.6mm]
            \bottomrule
        \end{tabular}
    \end{center}
    \end{small}
    \bigskip
    \caption{Demonstrates the pertinence of each CF-scenario towards each algorithm (left) and each metric (right). Note that we can use ranking metrics for explicit feedback, however, we only use MSE as a design choice and due to it's direct relevance.}
    \label{model_scenario_table}
    \vspace{-6mm} 
\end{table*}

\subsection{Sampling Strategies} \label{common_sampling_schemes}
Given a user-item CF dataset \dataset, we aim to create a $p\%$ subset $\mathcal{D}^{s, p}$ according to some sampling strategy $s$. In this paper, to be comprehensive, we consider a sample of eight popular sampling strategies, which can be grouped into the following three categories:

\subsubsection{Interaction sampling. \ \ } We first discuss three strategies that sample interactions from \dataset. In \emph{Random Interaction Sampling}, we generate $\mathcal{D}^{s, p}$ by randomly sampling $p\%$ of all the user-item interactions in \dataset. \emph{User-history Stratified Sampling} is another popular sampling technique (see \eg \cite{svae, handbook}) to generate smaller CF-datasets. To match the user-frequency distribution amongst \dataset and $\mathcal{D}^{s, p}$, it randomly samples $p\%$ of interactions from each user's consumption history. Unlike random stratified sampling, \emph{User-history Temporal Sampling} samples $p\%$ of the \emph{most recent} interactions for each user. This strategy is representative of the popular practice of making data subsets from the online traffic of the last $x$ days \cite{eclare, pfastre}.

\subsubsection{User sampling. \ \ } Similar to sampling interactions, we also consider two strategies which sample users in \dataset instead. To ensure a fair comparison amongst the different kinds of sampling schemes used in this paper, we retain exactly $p\%$ of the \emph{total interactions} in $\mathcal{D}^{s, p}$. In \emph{Random User Sampling}, we retain users from \dataset at random. To be more specific, we iteratively preserve \emph{all} the interactions for a random user until we have retained $p\%$ of the original interactions. Another strategy we employ is \emph{Head User Sampling}, in which we iteratively remove the user with the least amount of total interactions. This method is representative of commonly used data pre-processing strategies (see \eg \cite{mvae, neural_mf}) to make data suitable for parameter-heavy algorithms. Sampling the data in such a way can introduce bias toward users from minority groups which might raise concerns from a diversity and fairness perspective \cite{fairness}.

\subsubsection{Graph sampling. \ \ } Instead of sampling directly from \dataset, we also consider three strategies that sample from the inherent user-item bipartite interaction graph $\mathcal{G}$. In \emph{Centrality-based Sampling}, we proceed by computing the pagerank centrality scores \cite{pagerank} for each node in $\mathcal{G}$, and retain all the edges (interactions) of the \emph{top scoring nodes} until a total $p\%$ of the original interactions have been preserved. Another popular strategy we employ is \emph{Random-walk Sampling} \cite{large_graphs}, which performs multiple random-walks with restart on $\mathcal{G}$ and retains the edges amongst those pairs of nodes that have been visited at least once. We keep expanding our walk until $p\%$ of the initial edges have been retained. We also utilize \emph{Forest-fire Sampling} \cite{forest_fire}, which is a snowball sampling method and proceeds by randomly ``burning'' the outgoing edges of visited nodes. It initially starts with a random node, and then propagates to a random subset of previously unvisited neighbors. The propagation is terminated once we have created a graph-subset with $p\%$ of the initial edges.

\subsection{\sampler: Selection-Via-Proxy for CF data} \label{svp_cf} 
Selection-Via-Proxy (SVP) \cite{svp} is a leading coreset mining technique for classification datasets like CIFAR10 \cite{cifar} and ImageNet \cite{image_net}. The main idea proposed is simple and effective, and proceeds by training a relatively inexpensive base-model as a proxy to define the ``importance'' of a data-point. However, applying SVP to CF-data can be highly non-trivial because of the following impediments:
\begin{itemize}
    \listheader{Data heterogeneity:} Unlike classification data 
    over some input space $\mathcal{X}$ and label-space $\mathcal{Y}$, CF-data consists of numerous four-tuples $\{u, i, r^u_i, t^u_i\}$. Such multimodal data adds many different dimensions to sample data from, making it increasingly complex to define meaningful samplers. 
    
    \listheader{Defining the importance of a data point:} Unlike classification, where we can measure the performance of a classifier by 
    its
    empirical risk on held-out data, for recommendation, there are a variety of different scenarios (\cref{feedback_types}) along with a wide list of relevant evaluation metrics. Hence, it becomes challenging to adapt importance-tagging techniques like greedy k-centers \cite{k_centers}, forgetting-events \cite{forgetting_events}, \etc for recommendation tasks.
    
    \listheader{Missing data:} 
    CF-data is well-known for (1) 
    its
    sparsity; (2) skewed and long-tail user/item distributions; and (3) missing-not-at-random (MNAR) properties of the user-item interaction matrix. This results in additional problems as we are now sampling data from skewed, MNAR data, especially using proxy-models trained on the same skewed data. Such sampling in the worst-case might even lead to exacerbating existing biases in the data or even aberrant data samples.
\end{itemize}
To address these fundamental limitations in applying the SVP philosophy to CF-data, we propose \sampler to sample representative subsets from large user-item interaction data. \sampler is also specifically devised for our objective of benchmarking different recommendation algorithms, as it relies on the crucial assumption that the ``easiest'' part of a dataset will generally be easy \emph{for all} algorithms. Under this assumption, even after removing such data we are still likely to retain the overall algorithms' ranking.

Because of the inherent data heterogeneity in user-item interaction data, we can sub-sample in a variety of different ways. We design \sampler to be versatile in this aspect as it can be applied to sample users, items, interactions, or combinations of them, by marginally adjusting the definition of importance of each data-point. In this paper, we limit the discussion to only sampling users and interactions (separately), but extending \sampler for sampling across other data modalities should be relatively straightforward.

Irrespective of whether to sample users or interactions, \sampler proceeds by training an inexpensive proxy model $\mathcal{P}$ on the full, original data \dataset and modifies the forgetting-events approach \cite{forgetting_events} to retain the points with the \emph{highest} importance. To be more specific, for explicit feedback, we define the importance of each data-point \ie $\{u, i, r^u_i, t^u_i\}$ interaction as $\mathcal{P}$'s average MSE (over epochs) of the specific interaction if we're sampling interactions \emph{or} $\mathcal{P}$'s average MSE of $u$ (over epochs) if we're sampling users. Whereas, for implicit and sequential feedback, we use $\mathcal{P}$'s average inverse-AUC while computing the importance of each data-point. For the sake of completeness, we experiment with both Bias-only and MF as two different kinds of proxy-models for \sampler. Since both models can be trained for all three CF-scenarios (\cref{model_scenario_table}), we can directly use them to tag the importance for each CF-scenario.

Ultimately, to handle the MNAR and long-tail problems, we also propose \samplerprop which employs user and item propensities to correct the distribution mismatch while estimating the importance of each datapoint. More specifically, let $p_{u, i} = P(r^u_i = 1 ~|~ \overstar{r}^u_i = 1)$ denote the probability of user $u$ and item $i$'s interaction actually being observed (propensity), $E$ be the total number of epochs that $\mathcal{P}$ was trained for, $\mathcal{P}_e$ denote the proxy model after the $e^{\mathit{th}}$ epoch, $\mathcal{I}_u^+ \coloneqq \{ j ~|~ r^u_j > 0 \}$ be the set of positive interactions for $u$, and $\mathcal{I}_u^- \coloneqq \{ j ~|~ r^u_j = 0 \}$ be the set of negative interactions for $u$; then, the importance function for \samplerprop, $\mathcal{I}_p$ is defined as follows:
\begin{align*}
    \mathcal{I}_p(u ~|~ \mathcal{P}) \coloneqq \frac{1}{|\mathcal{I}_u^+|} \cdot \sum_{i \in \mathcal{I}_u^+} \mathcal{I}_p(u, i ~|~ \mathcal{P}) \hspace{0.63em} ; \hspace{0.63em}
    \mathcal{I}_p(u, i ~|~ \mathcal{P}) \coloneqq \frac{\Delta(u, i ~|~ \mathcal{P})}{p_{u, i}} \\
\end{align*}
\vspace{-0.6cm}
\begin{equation*}
\begin{gathered}
    \text{where,}
    \qquad \Delta(u, i ~|~ \mathcal{P}) \coloneqq 
    \begin{dcases} 
      ~~ \sum_{e=1}^{E} \left(\mathcal{P}_e(u, i) - r^u_i\right)^2 \\  
      ~~ \text{(for explicit feedback)} \\ \\
      ~~ \sum_{e=1}^{E} \sum_{j \sim \mathcal{I}_u^-} \frac{1}{\mathds{1}\left(\mathcal{P}_e(u, i) > \mathcal{P}_e(u, j)\right)} \\ 
      ~~ \text{(for implicit/sequential feedback)}
   \end{dcases}
\end{gathered}
\end{equation*}
\vspace{0.1cm}

\begin{proposition}
Given an ideal propensity-model $p_{u, i}; \ \mathcal{I}_p(u, i ~|~ \mathcal{P})$ is an unbiased estimator of $\Delta(u, i ~|~ \mathcal{P})$.
\end{proposition}

\begin{proof}
\begin{align*}
    \EE_{u \sim \mathcal{U}} \EE_{i \sim \mathcal{I}} \left[ \mathcal{I}_p(u, i ~|~ \mathcal{P}) \right] \hspace{5.65cm}
\end{align*}
\vspace{-0.4cm}
\begin{align*}
    &= \frac{1}{|\mathcal{U}| |\mathcal{I}|} \sum_{u \sim \mathcal{U}} \sum_{i \sim \mathcal{I}} \mathcal{I}_p(u, i ~|~ \mathcal{P}) \cdot P(r_i^u = 1) \\
    &= \begin{aligned}[t]
        \frac{1}{|\mathcal{U}| |\mathcal{I}|} \sum_{u \sim \mathcal{U}} \sum_{i \sim \mathcal{I}} \frac{\Delta(u, i ~|~ \mathcal{P})}{p_{u, i}} \cdot ( P(\overstar{r}_i^u = 0) \cdot \cancelto{0}{P(r_i^u = 1 ~|~ \overstar{r}_i^u = 0)} \\
        +~ P(\overstar{r}_i^u = 1) \cdot P(r_i^u = 1 ~|~ \overstar{r}_i^u = 1)) \\ 
    \end{aligned} \\
    &= \frac{1}{|\mathcal{U}| |\mathcal{I}|} \sum_{u \sim \mathcal{U}} \sum_{i \sim \mathcal{I}} \Delta(u, i ~|~ \mathcal{P}) \cdot P(\overstar{r}_i^u = 1) \\ 
    &= \EE_{u \sim \mathcal{U}} \EE_{i \sim \mathcal{I}} \left[ \Delta(u, i ~|~ \mathcal{P}) \right] \qedhere
\end{align*}
\end{proof} 

\paragraph{Propensity model.} A wide variety of ways exist 
to model the propensity score of a user-item interaction \cite{propensity_1, rec_as_treatments, sachdeva_kdd20, pfastre}. The most common ways comprise using machine learning models like na\"ive bayes and logistic regression \cite{rec_as_treatments}, or by fitting handcrafted functions \cite{pfastre}. For our problem statement, we make a simplifying assumption that the data noise is one-sided \ie $P(r^u_i = 1 ~|~ \overstar{r}^u_i = 0)$ or the probability of a user interacting with a \emph{wrong} item is \emph{zero}, and model the probability of an interaction going missing to decompose over the user and item as follows:
\begin{align*}
    p_{u, i} &= P(r^u_i = 1 ~|~ \overstar{r}^u_i = 1) \\
    &= P(r^u = 1 ~|~ \overstar{r}^u = 1) \cdot P(r_i = 1 ~|~ \overstar{r}_i = 1) ~=~ p_u \cdot p_i
\end{align*}
Ultimately, following \cite{pfastre}, we assume the user and item propensities to lie on the following sigmoid curves:
\begin{equation*}
\begin{split}
    p_u \coloneqq \frac{1}{1 + C_u \cdot e^{-A \cdot log(N_u + B)}} \quad ; \quad p_i \coloneqq \frac{1}{1 + C_i \cdot e^{-A \cdot log(N_i + B)}}
\end{split}
\end{equation*}
Where, $N_u$ and $N_i$ represent the total number of interactions of user $u$ and item $i$ respectively, $A$ and $B$ are two fixed scalars, $C_u = (log(|\mathcal{U}|) - 1) \cdot (B+1)^A$ and $C_i = (log(|\mathcal{I}|) - 1) \cdot (B+1)^A$. 

\subsection{Performance of a sampling strategy} \label{sampling_perf}
To quantify the performance of a sampling strategy $s$ on a dataset $\mathcal{D}$, we start by creating various $p\%$ subsets of $\mathcal{D}$ according to $s$ and call them $\mathcal{D}^{s, p}$. Next, we train and evaluate all the relevant recommendation algorithms on both $\mathcal{D}$ and $\mathcal{D}^{s, p}$. Let the \emph{ranking} of all algorithms according to CF-scenario $f$ and metric $m$ trained on $\mathcal{D}$ and $\mathcal{D}^{s, p}$ be $\mathcal{R}_{f, m}$ and $\mathcal{R}^{s, p}_{f, m}$ respectively, then the performance measure $\Psi(\mathcal{D}, s)$ is defined as the average correlation between $\mathcal{R}_{f, m}$ and $\mathcal{R}^{s, p}_{f, m}$ measured through Kendall's Tau over all possible CF-scenarios, metrics, and sampling percents:
\begin{equation*}
\begin{split}
    \Psi\left(\mathcal{D}, s\right) &= \lambda \cdot \sum_{f} \sum_{m} \sum_{p} \tau\left(\mathcal{R}_{f, m}, \mathcal{R}^{s, p}_{f, m}\right) \\
\end{split}
\end{equation*}
Where $\lambda$ is an appropriate normalizing constant for computing the average, sampling percent $p \in \{ 80, 60, 40, 20, 10, 1 \}$, CF-scenario $f$, metric $m$ and their pertinence towards each other can all be found in \cref{model_scenario_table}. $\Psi$ has the same range as Kendall's Tau \ie $[-1, 1]$ and a higher $\Psi$ indicates strong agreement between the algorithm ranking on the full and sub-sampled datasets, whereas a large negative $\Psi$ implies that the algorithm order was effectively reversed.

\subsection{Experiments} \label{main_exp}

\paragraph{Datasets.} To promote dataset diversity in our experiments, we use six public user-item rating interaction datasets with varying sizes, sparsity patterns, and other characteristics. We use the Magazine, Luxury, and Video-games categories of the Amazon review dataset \cite{amz_data}, along with the Movielens-100k \cite{movielens}, BeerAdvocate \cite{beer_dataset}, and GoodReads Comics \cite{mengting_goodreads} datasets. A brief set of data statistics is also presented in \cref{data_stats}. We simulate all three CF-scenarios (\cref{feedback_types}) 
via different pre-processing strategies. For explicit and implicit feedback, we follow a randomized 80/10/10 train-test-validation split for each user's consumption history in the dataset, and make use of the leave-one-last \cite{train_test_splitting} strategy for sequential feedback \ie
keep the last two interactions in each user's time-sorted consumption history in the validation and test-set respectively.
Since we can't control the initial construction of datasets, and to minimize the initial data bias, we follow the least restrictive data pre-processing \cite{making_progress, sigir20}. We only weed out the users 
with less
than 3 interactions,
to keep at least one occurrence 
in the train, validation, and test sets.

\begin{table}[!ht]
    \begin{footnotesize} 
    \begin{center}
        \begin{tabular}{c | c c c c}
            \toprule
            \multirow{2}{*}{\textbf{Dataset}} & \textbf{\#} & \textbf{\#} & \textbf{\#} & \textbf{Avg. User} \\ 
            & \textbf{Interactions} & \textbf{Users} & \textbf{Items} & \textbf{history length} \\
            \midrule
            
            Amazon Magazine      & 12.7k & 3.1k  & 1.3k  & 4.1 \\
            ML-100k              & 100k  & 943   & 1.7k  & 106.04 \\
            Amazon Luxury        & 126k  & 29.7k & 8.4k  & 4.26 \\
            Amazon Video-games   & 973k  & 181k  & 55.3k & 5.37 \\
            BeerAdvocate         & 1.51M & 18.6k & 64.3k & 81.45 \\
            Goodreads Comics     & 4.37M & 133k  & 89k   & 32.72 \\
            
            \bottomrule
        \end{tabular}
    \end{center}
    \end{footnotesize}
    \vspace{2mm}
    \caption{Data statistics of the \emph{six} datasets used in this paper.}
    \label{data_stats}
    \vspace{-6mm} 
\end{table}

\paragraph{Training details.} We implement all algorithms in PyTorch\footnote{Code is available at \href{https://github.com/noveens/sampling_cf}{\color{blue}{https://github.com/noveens/sampling\_cf}}} 
and train on a single GPU. For a fair comparison across algorithms, we perform 
hyper-parameter search on the validation set. 
For the three smallest datasets used in this paper (\cref{data_stats}), we search the latent size in $\{ 4, 8, 16, 32, 50 \}$, dropout in $\{ 0.0, 0.3, 0.5 \}$, and the learning rate in $\{ 0.001, 0.006, 0.02 \}$. Whereas for the three largest datasets, we fix the learning rate to be $0.006$. 
Note that despite the limited number of datasets and recommendation algorithms used in this study, given that we need to train all algorithms with hyper-parameter tuning for all CF scenarios, $\%$ data sampled according to all different sampling strategies discussed in \cref{common_sampling_schemes}, there are a total of 
$\sim400k$ 
unique models trained, 
equating to
a cumulative train time of over $400k ~ \times \sim1 \text{min} \approx 9$ months.

\paragraph{Data sampling.} To compute the $\Psi$-values as defined in \cref{sampling_perf}, we construct $\{ 80, 60, 40, 20, 10, 1 \} \%$ samples for each dataset and sampling strategy. To keep comparisons as fair as possible, for all sampling schemes, we only sample on the train set and never touch the validation and test sets. 

\subsubsection{How do different sampling strategies compare to each other? \ \ } Results with $\Psi$-values for all sampling schemes on all datasets are in \cref{psi_results}. Even though there are only six datasets under consideration, there are a few prominent patterns. First, the average $\Psi$ for most sampling schemes is around $0.4$, which implies a statistically significant correlation between the ranking of algorithms on the full \vs sub-sampled datasets. Next, \sampler generally outperforms all commonly used sampling strategies by some margin in retaining the ranking of different recommendation algorithms. Finally, the methods that discard the tail of a dataset (head-user and centrality-based) are the worst performing strategies overall, which supports the recent warnings against dense sampling of data \cite{sigir20}.

\newcommand{\STAB}[1]{\begin{tabular}{@{}c@{}}#1\end{tabular}}
\begin{table*}
    \begin{footnotesize}
    \begin{center}
        \begin{tabular}{c c | c c c c c c | c}
            \toprule
            \multicolumn{2}{c|}{\multirow{4}{*}{\textbf{Sampling strategy}}} & \multicolumn{6}{c|}{\emph{Datasets}} & \\
            & & \multicolumn{6}{c|}{} & \\
            & & \begin{tabular}{@{}c@{}}\textbf{Amazon}\\\textbf{Magazine}\end{tabular} & \textbf{ML-100k} & \begin{tabular}{@{}c@{}}\textbf{Amazon}\\\textbf{Luxury}\end{tabular} & \begin{tabular}{@{}c@{}}\textbf{Amazon}\\\textbf{Video-games}\end{tabular} & \textbf{BeerAdvocate} & \begin{tabular}{@{}c@{}}\textbf{Goodreads}\\\textbf{Comics}\end{tabular} & \textbf{\emph{Average}} \\
            \midrule
            
            \multirow{8}{*}{\STAB{\rotatebox[origin=c]{90}{\begin{tabular}{@{}c@{}}Interaction sampling\\\end{tabular}}}} & Random & 0.428     &  0.551     &  0.409     &  0.047     &  0.455     &  0.552     &  0.407 \\[0.6mm]
            & Stratified & 0.27      &  0.499     &  0.291     &  -0.01     &  0.468     &  0.538     &  0.343 \\[0.6mm]
            & Temporal & 0.289     &  0.569     &  0.416     &  -0.02     &  \underline{0.539}     &  0.634     &  0.405 \\[0.6mm]
            & \sampler \emph{w/} MF & 0.418     &  0.674     &  0.398     &  0.326     &  0.425     &  \underline{0.662}     &  \underline{0.484} \\[0.6mm]
            & \sampler \emph{w/} Bias-only & 0.38      &  0.684     &  \underline{0.431}     &  \underline{0.348}     &  0.365     &  0.6       &  0.468 \\[0.6mm]
            & \samplerprop \emph{w/} MF & 0.381     &  0.617     &  0.313     &  0.305     &  0.356     &  0.608     &  0.43 \\[0.6mm]
            & \samplerprop \emph{w/} Bias-only & 0.408     &  0.617     &  0.351     &  0.316     &  0.437     &  0.617     &  0.458 \\[0.6mm]
            \midrule
            \multirow{7}{*}{\STAB{\rotatebox[origin=c]{90}{\begin{tabular}{@{}c@{}}User sampling\\\end{tabular}}}} & Random & 0.436     &  0.622     &  0.429     &  0.17      &  0.344     &  0.582     &  0.431 \\[0.6mm]
            & Head & 0.369     &  0.403     &  0.315     &  0.11      &  -0.04     &  -0.02     &  0.19 \\[0.6mm]
            & \sampler \emph{w/} MF & 0.468     &  0.578     &  0.308     &  0.13      &  0.136     &  0.444     &  0.344 \\[0.6mm]
            & \sampler \emph{w/} Bias-only & 0.49      &  0.608     &  0.276     &  0.124     &  0.196     &  0.362     &  0.343 \\[0.6mm]
            & \samplerprop \emph{w/} MF & 0.438 &  0.683 &  0.307 &  0.098 &  0.458 &  0.592 &  0.429 \\[0.6mm]
            & \samplerprop \emph{w/} Bias-only & 0.434     &  \underline{0.751}     &  0.233     &  0.107     &  0.506     &  0.637     &  0.445 \\[0.6mm]
            \midrule
            \multirow{4}{*}{\STAB{\rotatebox[origin=c]{90}{\begin{tabular}{@{}c@{}}Graph\\\end{tabular}}}} & Centrality & 0.307     &  0.464     &  0.407     &  0.063     &  0.011     &  0.343     &  0.266 \\[0.6mm]
            & Random-walk & \underline{0.596}  &  0.5       &  0.395     &  0.306     &  0.137     &  0.442     &  0.396 \\[0.6mm]
            & Forest-fire & 0.564  &  0.493   &  0.415   &  0.265   &  0.099  &  0.454  &  0.382 \\[0.6mm]
            \bottomrule
        \end{tabular}
    \end{center}
    \end{footnotesize}
    \bigskip
    \caption{$\Psi$-values for all datasets and sampling strategies. Higher $\Psi$ is better. The best $\Psi$ for every dataset is \underline{underlined}. The $\Psi$-values for each sampling scheme \emph{averaged over all datasets} is appended to the right.}
    \label{psi_results}
    \vspace{-6mm} 
\end{table*}

\subsubsection{How does the relative performance of algorithms change as a function of sampling rate? \ \ } In an attempt to better understand the impact of sampling on different recommendation algorithms used in this study (\cref{algorithms}), we visualize the probability of a recommendation algorithm moving up in the overall method ranking with data sampling. We estimate the aforementioned probability using Maximum-Likelihood-Estimation (MLE) on the experiments already run in computing $\Psi(\mathcal{D}, s)$. Formally, given a recommendation algorithm $r$, CF-scenario $f$, and data sampling percent $p$:
\begin{equation*}
    P_{\mathit{MLE}}(r ~|~ f, p) = \lambda \cdot \sum_{\mathcal{D}} \sum_{s} \sum_{m} 0.5 + \frac{\mathcal{R}_{f, m}(r) - \mathcal{R}_{f, m}^{s, p}(r)}{2 \cdot (n-1)}
\end{equation*}
where $\lambda$ is an appropriate normalizing constant, and $n$ represents the total number of algorithms. A heatmap visualizing $P_{\mathit{MLE}}$ for all algorithms and CF-scenarios is shown in \cref{percent_sampling_vs_method}. We see that simpler methods like Bias-only and PopRec have the highest probability across data scenarios of increasing their ranking order with extreme sampling. Whereas parameter-heavy algorithms like SASRec, SVAE, MVAE, \etc tend to decrease in the ranking order, which is indicative of overfitting on smaller data samples.
\begin{figure}[ht!]     
    \centering
    \vspace{-0.1cm}
    \includegraphics[width=0.9\linewidth]{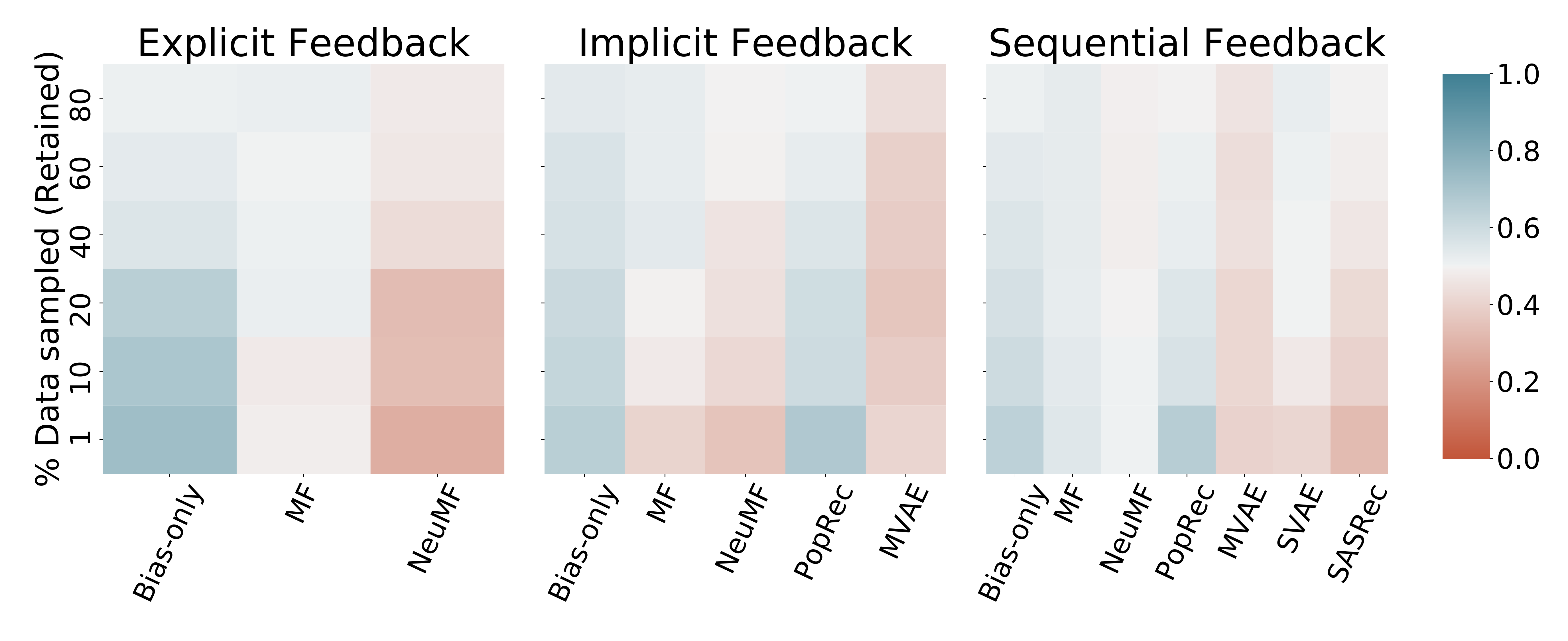}
    \vspace{-0.4cm}
    \caption{Heatmap of the probability of an algorithm moving in the overall ranking with extreme sampling. A high value indicates that the algorithm is most probable to \emph{move up} in the sampled data ranking order, whereas a low value indicates that the algorithm is most probable to \emph{move down}.}
    \label{percent_sampling_vs_method}
    \vspace{-0.3cm}
\end{figure}

\subsubsection{How much data to sample? \ \ } Since $\Psi$ is averaged over all $p \in \{ 80, 60, 40, 20, 10, 1 \}$\% data samples, to better estimate a reasonable amount of data to sample, we stratify $\Psi$ according to each value of $p$ and note the average Kendall's Tau. As we observe from \cref{percent_sampling_vs_tau}, 
there is a steady increase in the performance measure when more data is retained. Next, despite the results in \cref{percent_sampling_vs_tau} being averaged over \emph{sixteen} sampling strategies, we still notice a significant amount of performance retained after sampling just $50-60\%$ of the data. 
\begin{figure}[ht!] 
    \centering
    \vspace{-0.2cm}
    \includegraphics[width=0.6\linewidth]{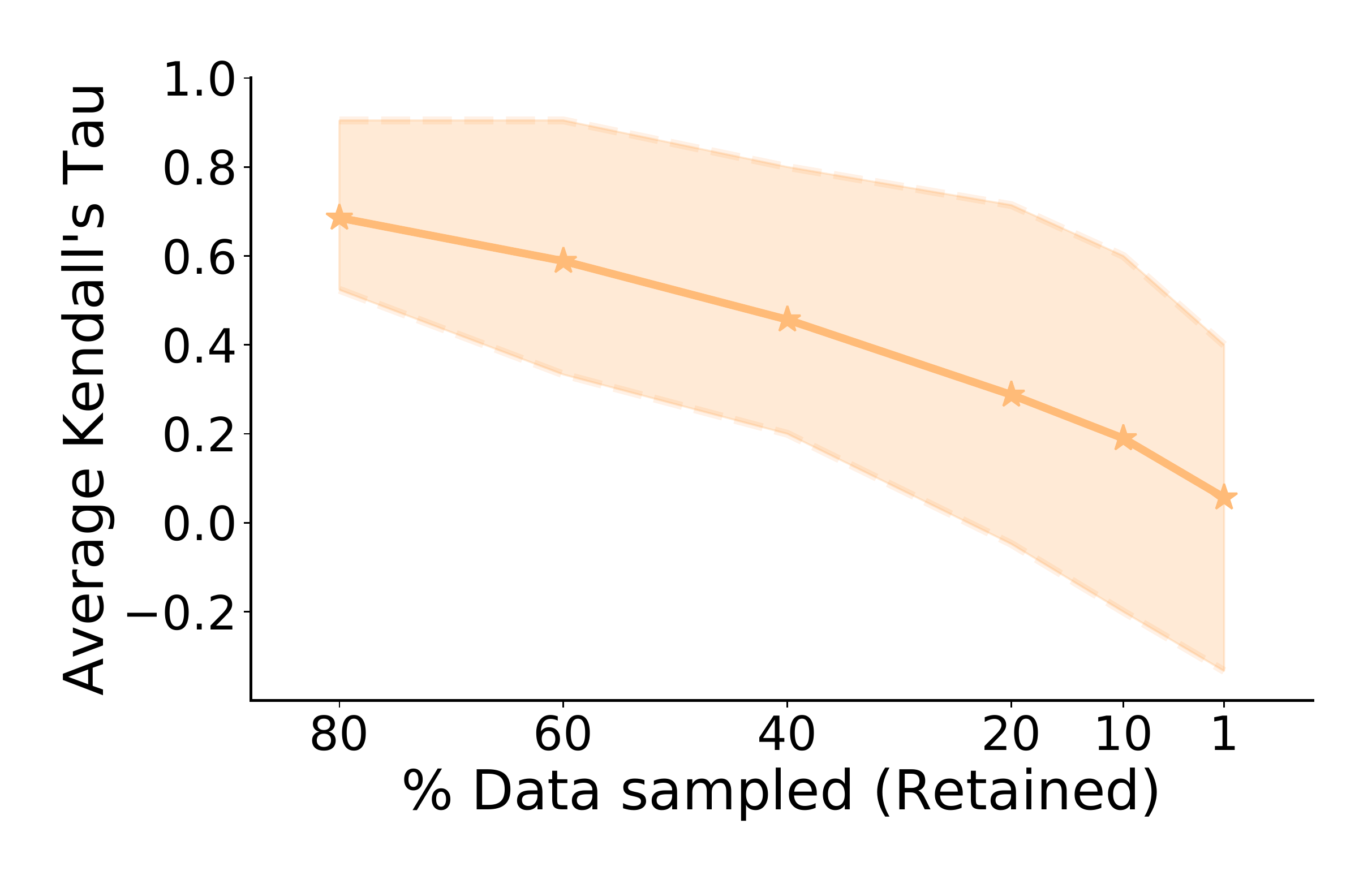}
    \vspace{-0.25cm}
    \caption{Comparison of the average Kendall's Tau with \% data sampled. A higher Tau indicates better retaining power of the ranking of different recommendation algorithms.}
    \label{percent_sampling_vs_tau}
    \vspace{-0.3cm}
\end{figure} 

\subsubsection{Are different metrics affected equally by sampling? \ \ } In an attempt to better understand how the different implicit and sequential feedback metrics (\cref{feedback_types}) are affected by sampling, we visualize the average Kendall's Tau for all sampling strategies (except \sampler for brevity) and all \% data sampling choices separately over the AUC, Recall, and nDCG metrics in \cref{metric_correlation}. As expected, we observe a steady decrease in the model quality across the accuracy metrics over the different sampling schemes. This is in agreement with the analysis from \cref{percent_sampling_vs_tau}. Next, most sampling schemes follow a \emph{similar} downwards trend in performance for the three metrics with AUC being slightly less affected and nDCG being slightly more affected by extreme sampling.
\begin{figure}[ht!] 
    \vspace{-0.1cm}
    \includegraphics[width=\linewidth]{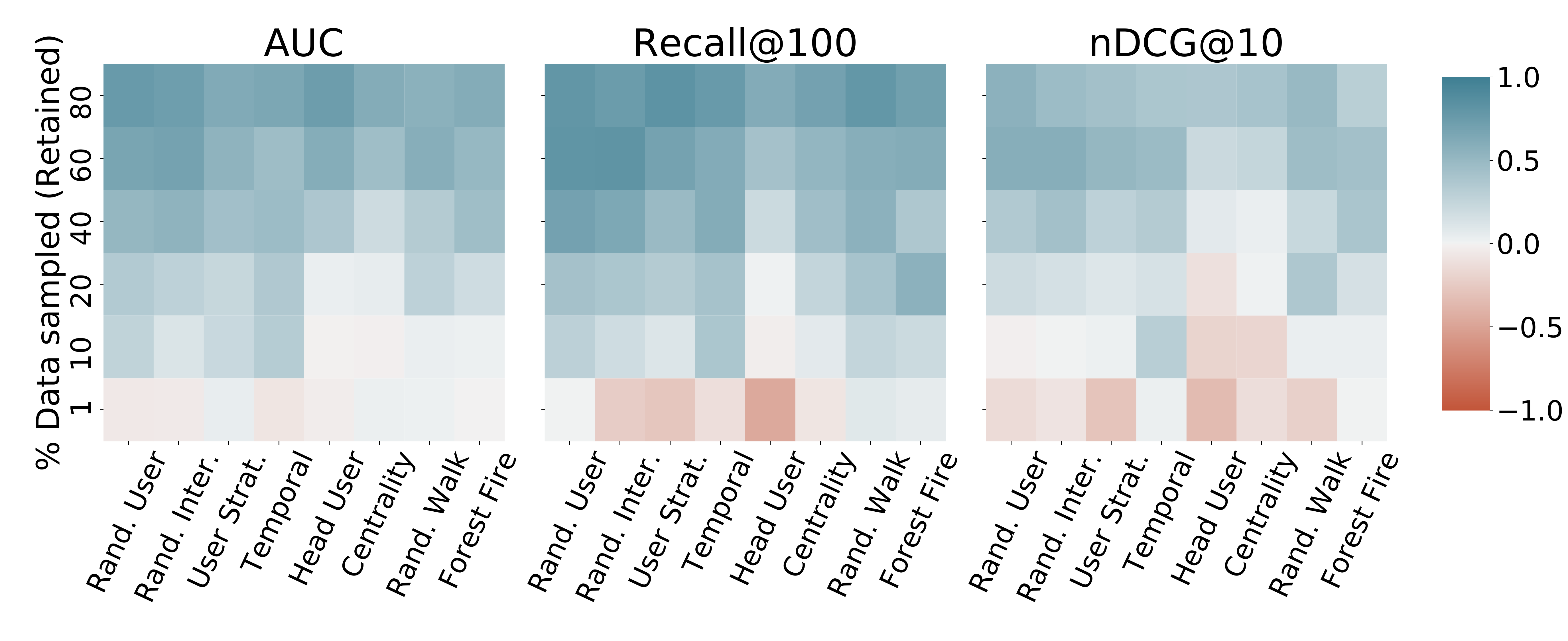}
    \vspace{-0.8cm}
    \caption{Heatmap of the average Kendall's Tau for different samplers stratified over metrics and \% data sampled.}
    \label{metric_correlation}
    \vspace{-0.5cm}
\end{figure}
\section{\oracle: Which sampler is best for me?} Although the results presented in \cref{main_exp} are indicative of correlation between the ranking of recommendation algorithms on the full dataset \vs smaller sub-samples, there still is no `one-size-fits-all' solution to the question of \emph{how to best sub-sample a dataset for retaining the performance of different recommendation algorithms?} In this section, we propose \oracle, that attempts to answer this question from a statistical perspective, in contrast with existing literature that generally has to resort to sensible heuristics \cite{large_graphs, scaling_up, sampling_cf_nn}.

\subsection{Problem formulation}
Given a dataset \dataset, we aim to gauge how a \emph{new} sampling strategy will perform in retaining the performance of different recommendation algorithms. Having already experimented with sixteen different sampling strategies on six datasets (\cref{main_exp}), we take a frequentist approach in predicting the performance of any sampling scheme. To be precise, to predict the performance of sampling scheme $s$ on dataset \dataset, we start by creating \dataset's subset according to $s$ and call it $\mathcal{D}^s$. We then represent \dataset and $\mathcal{D}^s$ in a low-dimensional latent space, followed by a powerful regression model to directly estimate the performance of $s$ on \dataset.

\subsection{Dataset representation} \label{data_rep}
We experiment with the following techniques of embedding a user-item interaction dataset into lower dimensions:

\subsubsection{Handcrafted. \ \ } For this method, we cherry-pick a few representative characteristics of \dataset and the underlying user-item bipartite interaction graph $\mathcal{G}$. Inspired by prior work \cite{large_graphs}, we represent \dataset as a combination of five features. We first utilize the frequency distribution of all users and items in \dataset. Next, we evaluate the distribution of the top$-100$ eigenvalues of $\mathcal{G}$'s adjacency matrix. All of these three distributions are generally long-tailed and heavily skewed. Furthermore, to capture notions like the diameter of $\mathcal{G}$, we compare the distribution of the number of hops $h$ \vs the number of pairs of nodes in $\mathcal{G}$ reachable at a distance less than $h$ \cite{hop_plot}. This distribution, unlike others is monotonically increasing in $h$. Finally, we also compute the size distribution of all connected components in $\mathcal{G}$, where a connected component is defined to be the maximal set of nodes, such that a path exists between any pair of nodes. Ultimately, we ascertain \dataset's final representation by concatenating $10$ evenly-spaced samples from each of the aforementioned distributions along with the total number of users, items, and interactions in \dataset. This results in a $53-$dimensional embedding for each dataset. Note that unlike previous work of simply \emph{retaining} the discussed features as a proxy of the quality of data subsets \cite{large_graphs}, \oracle instead uses these features to learn a regression model \emph{on-top} which can dynamically establish the importance of each feature in the performance of a sampling strategy.

\subsubsection{Unsupervised GCN. \ \ } With the recent advancements in the field of Graph Convolutional Networks \cite{original_gcn} to represent graph-structured data for a variety of downstream tasks, we also experiment with a GCN approach to embed $\mathcal{G}$. We modify the InfoGraph framework \cite{infograph}, which uses graph convolution encoders to obtain patch-level representations, followed by sort-pooling \cite{sort_pooling} to obtain a fixed, low-dimensional embedding for the entire graph. Since the nodes in $\mathcal{G}$ are the union of all users and items in \dataset, we randomly initialize $32-$dimensional embeddings using a Xavier-uniform prior \cite{xavier}. Parameter optimization is performed in an unsupervised fashion by maximizing the mutual information \cite{mutual_information} amongst the graph-level and patch-level representations of nodes in the same graph. We validate the best values of the latent dimension and number of layers of the GCN from $\{ 4, 8, 16, 32 \}$ and $\{ 1, 3 \}$ respectively.

\subsection{Training \& Inference} \label{oracle_architecture}
Having discussed different representation functions $\mathcal{E} : \mathcal{D} \mapsto \mathbb{R}^d$ to embed a CF-dataset in \cref{data_rep}, we now discuss \oracle's training framework agnostic to the actual details about $\mathcal{E}$. 

\paragraph{Optimization problem.} As a proxy of the performance of a sampler on a given dataset, we re-use the Kendall's Tau for each CF-scenario, metric, and sampling percent used while computing the $\Psi(\mathcal{D}, s)$ in \cref{main_exp}. To be specific, given $\mathcal{D}_{f}^{s, p}$ which is a $p\%$ sample of $f-$type feedback data $\mathcal{D}_f$, sampled according to sampling strategy $s$, we aim to estimate $\tau(\mathcal{R}_{f, m}, \mathcal{R}_{f, m}^{s, p})$ without ever computing the actual ranking of algorithms 
on either the full or sampled datasets:
\begin{equation} \label{tau_hat}
    \hat{\tau}\left(\mathcal{R}_{f, m}, \mathcal{R}_{f, m}^{s, p}\right) = \Phi\left(\mathcal{E}(\mathcal{D}_f), \mathcal{E}(\mathcal{D}_f^{s, p}), m\right),
\end{equation}
where $\Phi$ is an arbitrary neural network, and $m$ is the metric of interest (see \cref{model_scenario_table}). We train $\Phi$ by either (1) regressing on the Kendall's Tau computed for each CF scenario, metric, and sampling percent used while computing the $\Psi(\mathcal{D}, s)$ scores in \cref{main_exp}; or (2) performing BPR-style \cite{bpr} pairwise ranking on two sampling schemes $s_i \succ s_j \iff \tau(\mathcal{R}_{f, m}, \mathcal{R}_{f, m}^{s_i, p}) > \tau(\mathcal{R}_{f, m}, \mathcal{R}_{f, m}^{s_j, p})$. Formally, the two optimization problems are defined as follows:
\begin{align*}
    \argmin{\Phi} & \sum_{\mathcal{D}_f} \sum_{s} \sum_{p} \sum_{m} \left(\tau\left(\mathcal{R}_{f, m}, \mathcal{R}_{f, m}^{s, p}\right) - \hat{\tau}\left(\mathcal{R}_{f, m}, \mathcal{R}_{f, m}^{s, p}\right) \right)^2 \\ & \text{(\oracle-regression)} \\
    \argmin{\Phi} & \sum_{\mathcal{D}_f} \sum_{s_i \succ s_j} \sum_{p} \sum_{m} -\text{ln}~\sigma\left(\hat{\tau}\left(\mathcal{R}_{f, m}, \mathcal{R}_{f, m}^{s_i, p}\right) - \hat{\tau}\left(\mathcal{R}_{f, m}, \mathcal{R}_{f, m}^{s_j, p}\right) \right) \\ & \text{(\oracle-ranking)} \\
    \text{where, \;} & \hat{\tau}\left(\mathcal{R}_{f, m}, \mathcal{R}_{f, m}^{s, p}\right) ~=~ \Phi\left(\mathcal{E}(\mathcal{D}_f), \mathcal{E}(\mathcal{D}_f^{s, p}), m\right).
\end{align*}
The critical differences between the two aforementioned optimization problems are the downstream use-case and $\Phi$'s training time. If the utility of \oracle is to rank different sampling schemes for a given dataset, then \oracle-ranking is better suited as it is robust to the noise in computing $\tau(\mathcal{R}_{f, m}, \mathcal{R}_{f, m}^{s_i, p})$ like improper hyper-parameter tuning, local minima, \etc On the other hand, \oracle-regression is better suited for the use-case of estimating the exact values of $\tau$ for a sampling scheme on a given dataset. Even though both optimization problems converge in less than $2$ minutes given the data collected in \cref{main_exp}, the complexity of optimizing \oracle-ranking is still squared \wrt the total number of sampling schemes, whilst that of \oracle-regression is linear.

\paragraph{Architecture.} To compute $\Phi(\mathcal{E}(\mathcal{D}_f), \mathcal{E}(\mathcal{D}_f^{s, p}), m)$ we concatenate $\mathcal{E}(\mathcal{D}_f)$, $\mathcal{E}(\mathcal{D}_f^{s, p})$, one-hot embedding of $m$; and pass it through two relu-activated MLP projections to obtain $\hat{\tau}(\mathcal{R}_{f, m}, \mathcal{R}_{f, m}^{s, p})$. For \oracle-regression, we also pass the final output through a tanh activation, to reflect the range of Kendall's Tau \ie $[-1, 1]$.

\paragraph{Inference.} Since computing both $\mathcal{E}$ and $\Phi$ are agnostic to the datasets and the sampling schemes, we can simply use the trained $\mathcal{E}$ and $\Phi$ functions to rank \emph{any} sampling scheme for \emph{any} CF dataset. Computationally, given a trained \oracle, the utility of a sampling scheme can be computed simply by computing $\mathcal{E}$ twice, along with a single pass over $\Phi$, completing in the order of milliseconds.

\begin{figure*} \centering
\begin{minipage}{0.47\textwidth}
    \centering
    \captionsetup{type=table} 
    \begin{footnotesize} 
    \begin{center}
        \begin{tabular}{c c | c c c}
            \toprule
            $\mathcal{E}$ & $\Phi$ & \textbf{MSE} & \textbf{P@1} \\ \midrule
            
            \multicolumn{2}{c|}{Random} & 0.2336 & 25.2 \\ 
            \multicolumn{2}{c|}{User sampling \emph{w/} Bias-only \samplerprop} & -- & 30.6 \\ \midrule
            
            Handcrafted & Least squares regression & 0.1866 & 31.7 \\
            '' & XGBoost regression & 0.1163 & 43.9 \\
            '' & \oracle-regression & \underline{0.1008} & 51.2 \\ 
            '' & \oracle-ranking    & -- & 51.2 \\ \midrule
            
            Unsupervised GCN & Least squares regression & 0.1838 & 39.1 \\
            '' & XGBoost regression & 0.1231 & 43.9 \\
            '' & \oracle-regression & 0.1293 & 48.8 \\ 
            '' & \oracle-ranking    & -- & \underline{53.7} \\ \bottomrule
        \end{tabular}
    \end{center}
    \end{footnotesize}
    \vspace{0.5cm}
    \caption{Results for predicting the best sampling scheme for a particular dataset over a germane metric. The MSE-value next to randomly choosing the sampling scheme represents the variance of the test-set. Best values are \underline{underlined}.}
    \label{oracle_results}
    \vspace{-6mm} 
\end{minipage} \hspace{0.8cm} 
\begin{minipage}{0.46\textwidth}
    \centering
    \vspace{-0.5cm}
    \includegraphics[width=0.86\linewidth]{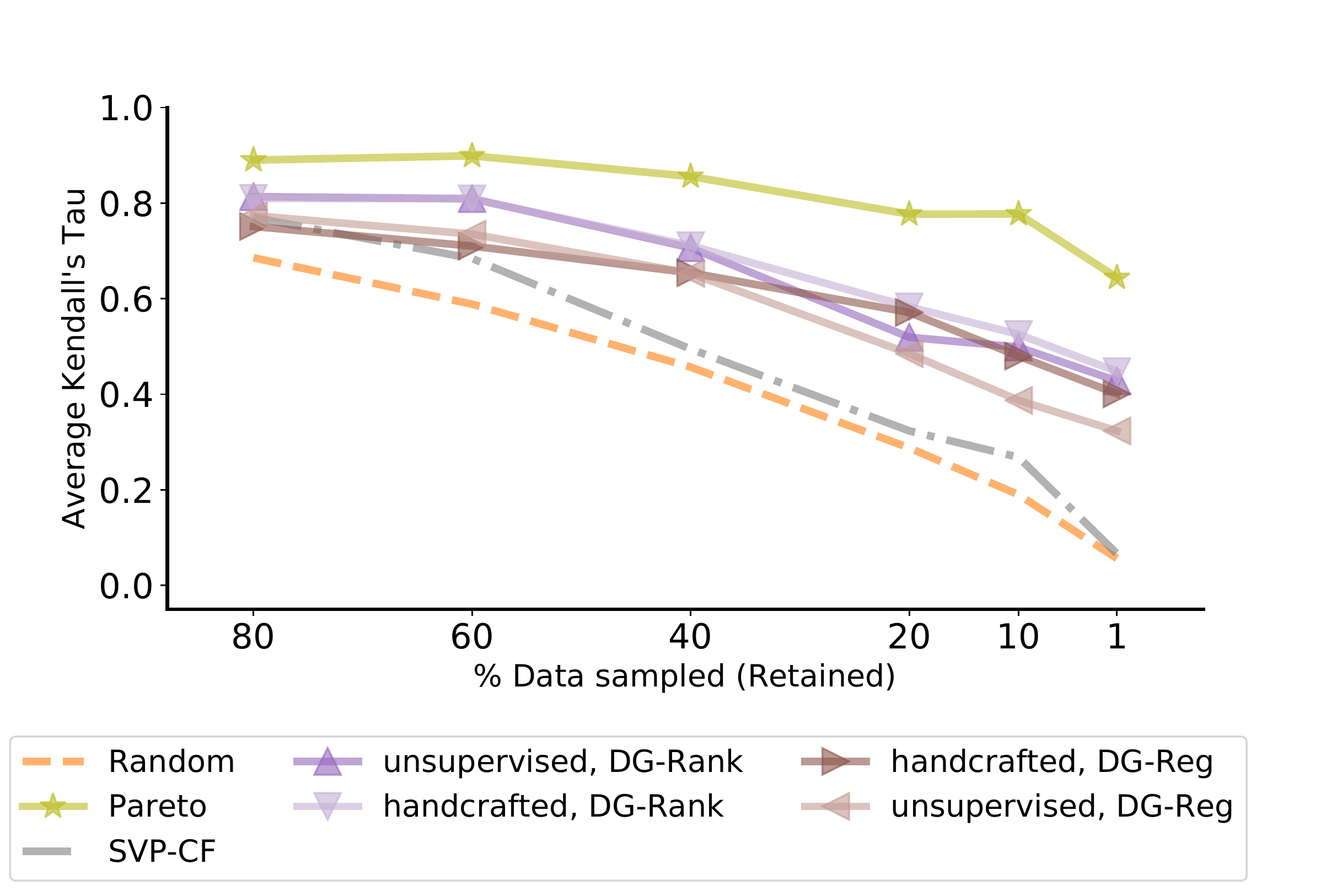}
    \vspace{-0.1cm}
    \caption{Comparison of the average Kendall's Tau with \% data sampled for different sampling-selection strategies. A higher Tau indicates better retaining power of the ranking of different recommendation algorithms.}
    \label{percent_sampling_vs_tau_oracle}
\end{minipage}
\end{figure*} 

\subsection{Experiments}
\paragraph{Setup.} 
We first create a train/validation/test split by randomly splitting all possible metrics and sampling $\%$ pairs $(m, p)$ into $70/15/15\%$ proportions. Subsequently for each dataset \dataset, CF-scenario $f$, and $(m, p)$ in the validation/test-set, we ask \oracle to rank all $16$ samplers (\cref{psi_results}) for $p\%$ sampling of $f-$type feedback for \dataset and use metric $m$ for evaluation by sorting $\hat{\tau}$ for each sampler, as defined in \cref{tau_hat}. To evaluate \oracle, we use the P$@1$ metric between the actual sampler ranking computed while computing $\Psi-$scores in \cref{main_exp}, and the one estimated by \oracle.

\subsubsection{How accurately can \oracle predict the best sampling scheme? \ \ } In \cref{oracle_results}, we compare all dataset representation choices $\mathcal{E}$, and multiple $\Phi$ architectures for the task of predicting the best sampling strategy. In addition to the regression and ranking architectures discussed in \cref{oracle_architecture}, we also compare with linear least-squares regression and XGBoost regression \cite{xgboost} as other choices of $\Phi$. In addition, we compare \oracle with simple baselines: (1) randomly choosing a sampling strategy; and (2) the best possible static sampler choosing strategy---always predict user sampling \emph{w/} Bias-only \samplerprop. First and foremost, irrespective of the $\mathcal{E}$ and $\Phi$ choices, \oracle outperforms both baselines. Next, both the handcrafted features and the unsupervised GCN features perform quite well in predicting the best sampling strategy, indicating that the graph characteristics are well correlated with the final performance of a sampling strategy. Finally, \oracle-regression and \oracle-ranking both perform better than alternative $\Phi-$choices, especially for the P$@1$ metric. 


\subsubsection{Can we use \oracle to sample more data without compromising performance? \ \ } In \cref{percent_sampling_vs_tau_oracle}, we compare the impact of \oracle in sampling more data by dynamically choosing an appropriate sampler for a given dataset, metric, and $\%$ data to sample. 
More specifically,
we compare the percentage of data sampled with the Kendall's Tau averaged over all datasets, CF-scenarios, and relevant metrics for different sampling strategy selection approaches. We compare \oracle with: (1) randomly picking a sampling strategy averaged over $100$ runs; and (2) the Pareto frontier as a skyline which always selects the best sampling strategy for any CF-dataset. As we observe from \cref{percent_sampling_vs_tau_oracle}, \oracle is better than predicting a sampling scheme at random, and 
is
much closer to the Pareto frontier. Next, pairwise ranking approaches are marginally better than regression approaches 
irrespective of $\mathcal{E}$. Finally, \oracle can appraise the best-performing recommendation algorithm with a suitable amount of confidence using only $10\%$ of the original data. This is significantly more efficient compared to having to sample $50-60\%$ if we were to always sample using a fixed strategy.
\section{Discussion} In this work, we discussed two approaches for representative sampling of CF-data, especially for accurately retaining the \emph{relative} performance of different recommendation algorithms. First, we proposed \sampler, which is better than commonly used sampling strategies, followed by introducing \oracle which \emph{analyzes} the performance of different samplers on different datasets. Detailed experiments suggest that \oracle can confidently estimate the downstream utility of any sampler within a few milliseconds, thereby enabling practitioners to benchmark different algorithms on $10$\% data sub-samples, with an average $5.8\times$ time speedup.

To realize the real-world environmental impact of \oracle, we discuss a typical weekly RecSys development cycle 
and its carbon footprint. 
Taking the Criteo Ad dataset as inspiration, we assume a common industry-scale dataset to have $\sim4$B interactions.
We assume a hypothetical use case that benchmarks for \eg $25$ different algorithms, each with $40$ different hyper-parameter variations. To estimate the energy consumption of GPUs, we scale the $0.4$ minute MLPerf \cite{mlperf} run of training NeuMF \cite{neural_mf} on the Movielens-20M dataset over an Nvidia DGX-2 machine. The total estimated run-time for all experiments would be $25 \times 40 \times \frac{4B}{20M} \times \frac{0.4}{60} \approx 1340$ hours; and following \cite{co2e}, the net CO$_2$ emissions would roughly be $10 \times 1340 \times 1.58 \times 0.954 \approx 20k$ lbs. To better understand the significance of this number, a brief CO$_2$ emissions comparison is presented in \cref{co2e}. Clearly, \oracle along with saving a large amount of experimentation time and cloud compute cost, can also significantly reduce the carbon footprint of this \emph{weekly process} by more than an average human's \emph{yearly} CO$_2$ emissions.

\begin{table}[!ht]
    \vspace{-5mm} 
    \begin{footnotesize} 
    \begin{center}
        \begin{tabular}{c c}
            \toprule
            \textbf{Consumption} & \textbf{CO$_2$e (lbs.)} \\ \midrule
            
            1 person, NY$\leftrightarrow$SF flight      & 2k \\
            Human life, 1 year avg.                     & 11k \\ 
            \midrule
            Weekly RecSys development cycle             & 20k \\
            '' \ \ \ \ \emph{w/} \oracle                & 3.4k \\
            
            \bottomrule
        \end{tabular}
    \end{center}
    \end{footnotesize}
    \vspace{2mm}
    \caption{CO$_2$ emissions comparison \cite{co2e}}
    \label{co2e}
    \vspace{-10mm} 
\end{table}

Despite having significantly benefited the run-time and environmental impacts of benchmarking 
algorithms on massive datasets, our analysis heavily relied on the experiments of training seven 
recommendation algorithms on six datasets and their various samples. Despite the already large experimental cost, we strongly believe that the downstream performance of \oracle could be further improved by simply 
considering
more algorithms and diverse datasets.
In addition to better sampling, analyzing the fairness 
aspects of training algorithms on sub-sampled datasets is an interesting research direction, which we plan to explore in future work.

\begin{acks}
This work was partly supported by NSF Award \#1750063.
\end{acks}

\bibliographystyle{ACM-Reference-Format}
\bibliography{acmart}




\end{document}